\documentclass[10pt, onecolumn, a4paper]{article}

\usepackage{graphicx}      
\usepackage{amsmath}
\usepackage{mathrsfs}
\usepackage{amsfonts}
\usepackage{amssymb}
\usepackage{color}
\usepackage{cancel}
\usepackage{subcaption}

\def\rea{\mathbb{R}}

\newtheorem{proposition}{Proposition}
\newtheorem{remark}{Remark}
\newtheorem{proof}{Proof}

\newtheorem{lemma}{Lemma}
\newtheorem{definition}{Definition}
\newtheorem{assumption}{Assumption}
\newcommand{\PAR}[2]{\frac{\partial{#1}}{\partial{#2}}}           
\newcommand{\MAP}[3]{{#1}:\mathbb{R}^{#2}\to\mathbb{R}^{#3}}      
\DeclareMathOperator{\sech}{sech}

\begin{document}

\title{Trajectory tracking for robotic arms with input saturation and only position measurements} 

\author{J. van der Veen, P. Borja, and J.M.A. Scherpen\\
	\textit{Jan C. Willems Center for Systems and Control},\\
		\textit{Engineering and Technology Institute Groningen (ENTEG)}\\
		University of Groningen, Groningen, The Netherlands. \\
		Email: \texttt{jochemvdveen@kpnmail.nl,}\\ \texttt{[l.p.borja.rosales, j.m.a.scherpen]@rug.nl}. 
}
\maketitle
\begin{abstract}                
In this work, we propose a passivity-based control approach that addresses the trajectory tracking problem for a class of mechanical systems that comprises a broad range of robotic arms. The resulting controllers can be naturally saturated and do not require velocity measurements. Moreover, the proposed methodology does not require the implementation of observers, and the structure of the closed-loop system permits the identification of a Lyapunov function, which eases the convergence analysis. To corroborate the effectiveness of the methodology, we perform experiments with the Philips Experimental Robot Arm.
\end{abstract}
%


\section{Introduction}


Customarily, for control purposes, robotic arms are modeled as mechanical systems \cite{spong2008robot,craig2009introduction,siciliano2016springer}, where energy-based modeling approaches, such as the Euler-Lagrange (EL) or the port-Hamiltonian (pH) one, have proven to be suitable to represent the behavior of these systems. In particular, the pH models underscore the roles of the energy, the interconnection pattern, and the dissipation of the system \cite{duindam2009modeling,VanDerSchaft2014}, which are the main components of passivity-based control (PBC) \cite{ortega2001putting}. Accordingly, PBC has demostrated to be a suitable methodology to control complex nonlinear mechanical systems. 

The literature on trajectory tracking for fully actuated mechanical systems is abundant. In particular, we refer the reader to \cite{zergeroglu2000global,chen2001global,borhaug2006global} and the references therein contained for a detailed exposition on this topic. However, the implementation of the methodologies that address this control problem is often hampered by the necessity of high gains to ensure stability, or some practical issues like the lack of sensors to measure velocities, and the necessity of consider bounded inputs to protect the actuators and ensure an appropriate performance of the controllers. Concerning the requirement of control laws without velocity measurements, in \cite{romero2014globally}, the authors propose a solution that ensures \textit{global uniform exponential} convergence towards the desired trajectories, where the velocities are estimated via and immersion and invariance observer. In \cite{dirksz2012tracking}, the authors propose a controller that achieves trajectory tracking with only position measurements, where a dynamic extension is implemented to remove the necessity of velocity measurements.  On the other hand, in \cite{AGUI}, the authors propose saturated control laws that guarantee \textit{global uniform asymptotic} convergence towards the desired trajectories. However, to this end, velocity measurements are needed. Other interesting results are reported in \cite{loria1998bounded}, where the authors propose controllers that achieve \textit{semi-global uniform asymptotic} convergence to the desired trajectories without velocity measurements and, simultaneously, ensuring that the inputs are bounded.

The main contribution of this work is the proposition of a PBC approach to address the trajectory tracking problem for a class of fully actuated mechanical systems while considering saturated inputs and without velocity measurements. To this end, we extend the controllers reported in \cite{WESBOR}, which are devised to solve the set-point regulation problem. Notable, such an extension is far from being trivial as in the trajectory tracking problem, the closed-loop system is nonautonomous. Hence, the stability analysis cannot be conducted via La Salle's arguments as in the mentioned reference. The proposed methodology offers an alternative to achieve trajectory tracking while considering physical limitations in the sensors and actuators of the system under study that are often obviated in the literature. Below, we present some of the main differences of our approach regarding \cite{romero2014globally,dirksz2012tracking,AGUI,loria1998bounded}, where similar problems have been studied.
\begin{itemize}
 \item Compared to \cite{romero2014globally,dirksz2012tracking,AGUI}, the proposed methodology tackles down both problems--input saturation and no velocity measurements--simultaneously.
 \item The necessity of measuring the velocities is removed via the implementation of a dynamic extension. Accordingly, in contrast to \cite{romero2014globally}, no observer is required, simplifying the convergence proof. 
 \item In comparison to \cite{dirksz2012tracking}, we propose a specific change of coordinates to identify a different Lyapunov function that allows us to claim global--instead of semi-global--convergence results.
 \item We adopt the pH approach instead of the EL one, contrasting to \cite{AGUI,loria1998bounded}, which allows us to identify a Lyapunov function straightforwardly. Furthermore, such a function is radially unbounded, which is instrumental in claiming the \textit{global uniform asymptotic} convergence of the closed-loop system trajectories.
 \item Our controller is composed of fewer elements than the one reported in \cite{loria1998bounded}, which may be favorable during the gain tuning process. 
 \end{itemize}

The remainder of this paper is organized as follows: in Section \ref{sec:prel}, we provide the preliminaries and formulate the problem under study. Section \ref{sec:cd} is devoted to the design of passivity-based controllers that address the trajectory tracking problem. In Section \ref{sec:implementation}, we illustrate the applicability of the proposed methodology via its implementation in the PERA system. Finally, we conclude this paper with some concluding remarks and future work in Section \ref{sec:conclusion}. 

\section{Preliminaries and problem formulation}\label{sec:prel}
This section discusses the mathematical modeling of the PERA system in the pH framework. The problem and objective of this work are defined. Moreover, the partial linearization via change of coordinates (PLvCC) that is necessary for the design of the control laws is explained. 

\subsection{pH representation of fully actuated mechanical systems}
Throughout this work, we consider mechanical systems that can be represented by the following pH model 
\begin{equation}\label{eq:porthamiltonianframework}
 \arraycolsep=1.6pt
\def\arraystretch{1.5}
\begin{array}{rcl}
 \begin{bmatrix}
  \dot{q} \\ \dot{p}
 \end{bmatrix}&=& \begin{bmatrix}
         {\mathbf{0}_{n\times n}} & {I_{n}} \\ {-I_{n}} & {\mathbf{0}_{n\times n}}          
                  \end{bmatrix}\begin{bmatrix}
                                \PAR{H}{q}(q,p) \\ \PAR{H}{p}(q,p)
                               \end{bmatrix}
+\begin{bmatrix}
            {\mathbf{0}_{n\times n}} \\ I_{n}                     
                                \end{bmatrix}u,\\
 H(q, p)&=&\frac{1}{2} p^{\top} M^{-1}(q) p+V(q),\\
 y&=& M^{-1}(q)p = \dot{q},
\end{array}
\end{equation}
where $q,p\in \mathbb{R}^n$ denote the generalized positions and momenta, respectively, $u\in\rea^{n}$ is the input, $\MAP{M}{n}{n\times n}$ is the inertia matrix, which is positive definite, $V:\rea^{n}\to\rea_{+}$ denotes the potential energy of the system, $H:\rea^{n}\times\rea^{n}\to\rea_{+}$ is the system's Hamiltonian, and $y\in\rea^{n}$ is the \textit{passive} output.

The following assumptions characterize the class of systems for which the methodology introduced in Section \ref{sec:cd} is suitable.

\begin{assumption}\label{ass1}
 The term $\PAR{V}{q}(q)$ is \textit{bounded} from above and from below.
\end{assumption}
\begin{assumption}\label{ass2}
 If $q$ is \textit{bounded}, then
 \begin{equation}
\begin{array}{rcl}
\left\lVert M(q) \right\rVert<\infty,&&
\left\lVert \displaystyle\frac{dM}{dt}(q)\right\rVert<\infty, \end{array} 
 \end{equation} 
 $\lVert (\cdot) \rVert$ denotes the spectral norm of a matrix.
\end{assumption}

Before presenting the formulation of the problem under study, we introduce the following definition.

\begin{definition}[Feasible trajectory]
 A trajectory $q=q_{d}(t)$ is \textit{feasible} if there exists a control input $u=u_{d}(t)$ such that the pair $(q_{d}(t),u_{d}(t))$ solves \eqref{eq:porthamiltonianframework}.  
\end{definition}

In this work, we assume that the desired trajectories $q_{d}(t), \ \dot{q}_{d}(t)$ are smooth and bounded.\\[0.1cm] 
\textbf{Problem setting.} Given a desired feasible trajectory $q_{d}(t)$, find a control law such that the trajectories of \eqref{eq:porthamiltonianframework} converge to $q_{d}(t)$, and the corresponding $p_{d}(t)$, while ensuring that: 
\begin{itemize}
    \item[\textbf{O1}] the control law does not depend on $p$. 
    \item[\textbf{O2}] the control signals satisfy $u_{i}(t)\in[\mathcal{U}_{\tt min},\mathcal{U}_{\tt max}]$ for all $t\geq0$, with $i=1,\dots,n$, and the constants $\mathcal{U}_{\tt min},\mathcal{U}_{\tt max}$ verify $\mathcal{U}_{\tt min} < \mathcal{U}_{\tt max}$.
\end{itemize}

\begin{remark}
 The class of systems that can be represented by \eqref{eq:porthamiltonianframework}, and verify Assumptions \ref{ass1} and \ref{ass2}, encompasses a broad range of robotic arms, for instance, fully actuated manipulators without non-holonomic constraints and only revolute joints.
 For further discussion on this topic, see \cite{spong2008robot}.
%
\end{remark}

\subsection{PLvCC of pH systems}
Let $\MAP{\Psi}{n}{n\times n}$ be a factor of $M^{-1}(q)$, i.e.,
\begin{equation}\label{eq:choleskyconfirm}
    M^{-1}(q)=\Psi (q) \Psi^\top (q).
\end{equation}
Notice that, since $M^{-1}(q)$ has \textit{full rank}, $\Psi(q)$ has full rank as well.
Define the new coordinates $\mathtt{P}:=\Psi^\top (q) p$. Then, \eqref{eq:porthamiltonianframework} can be rewritten as
\begin{equation}\label{eq:systemtransform}
 \arraycolsep=1.6pt
\def\arraystretch{1.5}
\begin{array}{rcl}
 \begin{bmatrix}
  \dot{q} \\ \dot{\mathtt{P}}
 \end{bmatrix}&=& \begin{bmatrix}
      {\mathbf{0}_{n\times n}} & {\Psi (q)} \\ {-\Psi^\top (q)} & {J(q,\mathtt{P})}             
                  \end{bmatrix}\begin{bmatrix}
                                \PAR{\bar{H}}{q}(q,\mathtt{P}) \\ \PAR{\bar{H}}{\mathtt{P}}(q,\mathtt{P})
                               \end{bmatrix} + \begin{bmatrix}
            {\mathbf{0}_{n\times n}} \\ \Psi^\top (q)                     
                                \end{bmatrix}u \\
\bar{H}(q,\mathtt{P})&=&\frac{1}{2}\mathtt{P}^\top\mathtt{P}+V(q) \\ 
y &=&\Psi(q) \frac{\partial \bar{H}}{\partial \mathtt{P}}(q,\mathtt{P})= \dot{q}
\end{array}
\end{equation}
where $J:\rea^{n}\times\rea^{n}\to\rea^{n\times n}$ is a \textit{skew-symmetric matrix} representing the gyroscopic forces present in the system \cite{romero2014globally}, and whose elements are given by
\begin{equation}\label{eq:gyroscopicforces}
    J_{ij}(q,\mathtt{P}) = -\mathtt{P}^\top \Psi^{-1}(q) \left[\Psi_i(q), \Psi_j(q)\right], 
\end{equation}
where $[\cdot,\cdot]$ denotes the standard Lie bracket \cite{SPI}. For a thorough exposition of PLvCC, we refer the reader to \cite{venkatraman2010speed}.
\subsection{Barbalat's lemma}
The stability proofs contained in Section \ref{sec:cd} are based on Barbalat's lemma, which we present below to ease the readability of this paper.
\begin{lemma}\label{barb}
 Consider a function $f:\rea\to\rea$ \textit{uniformly continuous} on the interval $[0,\infty)$. Suppose that 
 \begin{equation*}
  \displaystyle\lim_{t\to\infty}\displaystyle\int_{0}^{t}f(\tau)d\tau = \phi <\infty.
 \end{equation*} 
 Then, $f(t)\to 0$ as $t\to\infty$.
\end{lemma}
The proof of Lemma \ref{barb} may be found in \cite{KHA}.
\section{Control design}\label{sec:cd}

This section is devoted to the control design, where the main idea is to split the controller into two parts: (i) a control signal that lets us express the dynamics of the errors, between the system's trajectories and the desired ones, as a pH system. Then, following the results reported in \cite{WESBOR}, (ii) a controller that renders the origin of that pH system globally uniformly asymptotically stable while satisfying \textbf{O1} and \textbf{O2}.
\subsection{Boundedness}
When dealing with nonautonomous systems, particularly when applying Barbalat's lemma, it is fundamental to prove that the functions involved are bounded. Hence, before proceeding with the control design, we present the following lemma, which is instrumental for the stability proofs contained in this section.
\begin{lemma}\label{lem1}
 Assume that $q$ is bounded and system \eqref{eq:porthamiltonianframework} verifies Assumption \ref{ass2}. Then:
 \begin{itemize}
  \item [\textbf{(i)}] \begin{equation}
\left\lVert \displaystyle\frac{d\Psi}{dt}(q) \right\rVert<\infty,  
 \end{equation} 
  \item [\textbf{(ii)}] \begin{equation}
\left\lVert \displaystyle\frac{d\Psi^{-1}}{dt}(q) \right\rVert<\infty. 
 \end{equation} 
 \end{itemize}
\end{lemma}
\begin{proof}
Note that 
\begin{equation}
 \left\lVert M(q) \right\rVert<\infty \iff \left\lVert M^{-1}(q) \right\rVert<\infty,
\end{equation} 
and, from \eqref{eq:choleskyconfirm}, we get
\begin{equation}
 M(q) = \Psi^{-\top}(q)\Psi^{-1}(q).
\end{equation}
Thus,
\begin{equation}
 \arraycolsep=1.6pt
\def\arraystretch{1.5}
 \begin{array}{rcccl}
  \left\lVert M(q) \right\rVert<\infty &\iff&  \left\lVert \Psi^{-1}(q) \right\rVert<\infty &\iff& \left\lVert \Psi(q) \right\rVert<\infty.
 \end{array}\label{Psibound}
\end{equation} 
Now, to prove \textbf{(i)}, note that
\begin{equation*}
 \displaystyle\frac{dM^{-1}}{dt}(q) = -M^{-1}(q)\left( \displaystyle\frac{dM}{dt}(q) \right)M^{-1}(q).
\end{equation*} 
Accordingly, if Assumption \ref{ass2} holds, we get
\begin{equation}
 \left\lVert \displaystyle\frac{dM^{-1}}{dt}(q) \right\rVert < \infty. \label{Minvbound}
\end{equation} 
Moreover, 
\begin{equation*}
 \displaystyle\frac{dM^{-1}}{dt}(q) = \left( \displaystyle\frac{d\Psi}{dt}(q) \right)\Psi^{\top}(q) + \Psi(q)\left(  \displaystyle\frac{d\Psi^{\top}}{dt}(q) \right).
\end{equation*} 
Therefore, \eqref{Minvbound} holds only if \textbf{(i)} is satisfied.

To prove \textbf{(ii)} notice that, from \eqref{Psibound}, we get that
\begin{equation*}
 \displaystyle\frac{dM}{dt}(q) = \left( \displaystyle\frac{d\Psi^{-\top}}{dt}(q) \right)\Psi^{-1}(q) + \Psi^{-\top}(q)\left( \displaystyle\frac{d\Psi^{-1}}{dt}(q) \right)
\end{equation*} 
the spectral norm of which is bounded only if \textbf{(ii)} holds.
\end{proof}
\subsection{Desired dynamics and error system}
Given a feasible $q_{d}(t)$, the first step in the control design consists in defining the desired dynamics of the system, which are given by\footnote{To avoid cluttering in the notation, henceforth, we omit the argument $t$ from the desired trajectories.}
\begin{equation}\label{eq:systemdesired}
 \arraycolsep=1.6pt
\def\arraystretch{1.5}
\begin{array}{rcl}
{\dot{q}_d}&=& {\Psi (q)}\mathtt{P}_d \\
{\dot{\mathtt{P}}_d}&=& -\Psi^\top (q)\PAR{V}{q}(q_{d}) + J(q,\mathtt{P})\mathtt{P}_d + \Psi^\top (q)u_{d}
   \end{array}
\end{equation}
where we compute the desired control input $u_{d}$ by fixing $q=q_{d}, \mathtt{P}=\mathtt{P}_{d}$, and their corresponding time derivatives, in \eqref{eq:systemtransform}. Hence,
\begin{equation}\label{eq:ud}
\arraycolsep=1.6pt
\def\arraystretch{1.5}
\begin{array}{rcl}
    u_d &=& \Psi^{-\top}(q_d)\left[\frac{d }{d t}\left(\Psi^{-1}(q_d)\dot{q}_d\right) - J_{d}(t) \Psi^{-1}(q_d)\dot{q}_d\right]+\displaystyle\frac{\partial V}{\partial q}(q_d),
    \end{array}
\end{equation}
where the elements of the matrix $J_{d}(t)$ are given by
\begin{equation}
 J_{d_{ij}}(t) = -\dot{q}^{\top}_{d}M(q_{d})[\Psi_{i}(q_{d}),\Psi_{j}(q_{d})].
\end{equation} 
Notice that, Lemma \ref{lem1}, ensures that $u_{d}$ is bounded.

The next step in the control design is to transform the tracking problem into a stabilization one. Towards this end, we define the errors
\begin{equation}\label{errors}
 \begin{array}{rcl}
  \tilde{q}:=q-q_d, & \widetilde{\mathtt{P}}=\mathtt{P}-\mathtt{P}_d, & \tilde{u}=u-u_d.
 \end{array}
\end{equation} 
Therefore, from \eqref{eq:systemtransform}, \eqref{eq:systemdesired}, and \eqref{errors}, we get
\begin{equation}\label{eq:systemerror1}
 \begin{array}{rcl}
  \dot{\tilde{q}} &=& {\Psi (q)}{\widetilde{\mathtt{P}}}\\
  \dot{\widetilde{\mathtt{P}}}&=&\Psi^\top (q)\left[ \tilde{u} + \frac{\partial V}{\partial q}(q_d) - \frac{\partial V}{\partial q}(q) \right] + J(q,\mathtt{P}){\widetilde{\mathtt{P}}}
 \end{array}
\end{equation} 
Now, to express \eqref{eq:systemerror1} as a pH system, we fix
\begin{equation}\label{eq:utilde}
    \tilde{u}=\frac{\partial V}{\partial q}(q)-\frac{\partial V}{\partial q}(q_d)+\hat{u}.
\end{equation}
Thus, replacing \eqref{eq:utilde} in \eqref{eq:systemerror1} yields
\begin{equation}\label{eq:systemerror}
\arraycolsep=1.6pt
\def\arraystretch{1.5}
\begin{array}{rcl}
 \begin{bmatrix}
  \dot{\tilde{q}} \\ \dot{\widetilde{\mathtt{P}}} 
 \end{bmatrix} & = & \begin{bmatrix}
{\mathbf{0}_{n\times n}} & {\Psi (q)} \\ -\Psi^\top (q) & J(q,\mathtt{P})
                                        \end{bmatrix}
\begin{bmatrix}
 {\frac{\partial \widetilde{H}}{\partial \tilde{q}}(\widetilde{\mathtt{P}})} \\ {\frac{\partial \widetilde{H}}{\partial \widetilde{\mathtt{P}}}(\widetilde{\mathtt{P}})}
\end{bmatrix}+ \begin{bmatrix}
            {\mathbf{0}_{n\times n}} \\ \Psi^\top (q)                     
                                \end{bmatrix}\hat u \\
\widetilde{H}(\widetilde{\mathtt{P}})&=&\frac{1}{2}\widetilde{\mathtt{P}}^\top \widetilde{\mathtt{P}}\\
\tilde{y} &=&\Psi(q) \frac{\partial \widetilde{H}}{\partial \widetilde{\mathtt{P}}}(\widetilde{\mathtt{P}})=\dot{\tilde{q}}
\end{array}
\end{equation}
Note that by designing $\hat u$, in \eqref{eq:systemerror}, such that the closed-loop system has a uniformly asymptotically stable equilibrium at $(\tilde{q}, \widetilde{\mathtt{P}}) = (\mathbf{0}_{n},\mathbf{0}_{n})$, we guarantee that $q\to q_{d}, \mathtt{P}\to \mathtt{P}_{d}$ as $t\to\infty$.
\subsection{Control without velocity measurements}
The asymptotic stabilization problem of \eqref{eq:systemerror} may be addressed by performing an energy-shaping plus damping injection process. Nevertheless, the latter requires information---measurements---of $\widetilde{\mathtt{P}}$, and consequently of $\dot{q}$, which is often a nonmeasurable signal. To overcome this issue, we propose the controller state $x_{c}\in\rea^{n}$ with dynamics
\begin{equation}
    \dot{x}_c=-R_c\left(K_I z+K_c x_c\right)\label{eq:xc1} ,
\end{equation}
where $e_{i}$ denotes an element of the canonical basis of $\rea^{n}$, the matrices $R_{c},K_{c},K_{I}\in\rea^{n\times n}$ are positive definite, and $z\in\rea^{n}$ is defined as
\begin{equation}\label{eq:z}
 z(\tilde{q},x_{c}):=\tilde{q}+x_{c}.
\end{equation} 
The following proposition provides a controller that solves the global uniform asymptotic stabilization problem for \eqref{eq:systemerror} without velocity measurements.
\begin{proposition}\label{controllaw1}
Consider the augmented state vector $[\tilde{q}^{\top}, \widetilde{\mathtt{P}}^{\top},x_{c}^{\top}]^{\top}$, with dynamics \eqref{eq:systemerror1}-\eqref{eq:xc1}. Then, the control law
\begin{equation}\label{eq:uhat1}
    \hat{u}=-K_I z
\end{equation}
ensures that $(\tilde{q}_{*},\widetilde{\mathtt{P}}_{*},x_{c_{*}}) =(\mathbf{0}_{n},\mathbf{0}_{n},\mathbf{0}_{n})$ is a \textit{globally uniformly asymptotically stable equilibrium} point for the closed-loop system with Lyapunov function\footnote{We omit the argument $(\tilde{q},x_{c})$ from $z$ to simplify the notation.}
\begin{equation}\label{eq:hamiltonian1}
    \widetilde{H}_d(\tilde{q},\widetilde{\mathtt{P}},x_c)=\frac{1}{2}z^\top K_I z + \frac{1}{2}\widetilde{\mathtt{P}}^\top \widetilde{\mathtt{P}} + \frac{1}{2}x_c^\top K_c x_c .
\end{equation}
\end{proposition}
\begin{proof}
Note that
\begin{equation}
\arraycolsep=1.6pt
\def\arraystretch{1.5}
\begin{array}{rcl}
 \widetilde{H}_d(\mathbf{0}_{n},\mathbf{0}_{n},\mathbf{0}_{n}) &=& 0 \\
 \widetilde{H}_d(\tilde{q},\widetilde{\mathtt{P}},x_c) &>& 0, \; \forall \ \rea^{n}\times\rea^{n}\times\rea^{n}-\{ \mathbf{0}_{n},\mathbf{0}_{n},\mathbf{0}_{n}\}.
\end{array} \label{eq:hamiltonian1pos}
\end{equation} 
Thus, $\widetilde{H}_{d}(\tilde{q},\widetilde{\mathtt{P}},x_c)$ is \textit{positive definite} with respect to the equilibrium. 
Furthermore, replacing \eqref{eq:uhat1} in \eqref{eq:systemerror}, the dynamics of the augmented state vector take the form
 \begin{equation*}\label{eq:syserrcl}
 \begin{bmatrix}
  \dot{\tilde{q}} \\[0.15cm] \dot{\widetilde{\mathtt{P}}} \\[0.15cm] \dot{x}_{c} 
 \end{bmatrix}  =  \begin{bmatrix}
{\mathbf{0}_{n\times n}} & {\Psi (q)} & {\mathbf{0}_{n\times n}} \\[0.15cm] -\Psi^\top (q) & J(q,\mathtt{P}) & {\mathbf{0}_{n\times n}} \\[0.15cm] {\mathbf{0}_{n\times n}} & {\mathbf{0}_{n\times n}} & -R_{c}
                                        \end{bmatrix}
\begin{bmatrix}
 \frac{\partial \widetilde{H}_{d}}{\partial \tilde{q}}(\tilde{q},\widetilde{\mathtt{P}},x_c) \\[0.15cm] \frac{\partial \widetilde{H}_{d}}{\partial \widetilde{\mathtt{P}}}(\tilde{q},\widetilde{\mathtt{P}},x_c) \\[0.15cm] \frac{\partial \widetilde{H}_{d}}{\partial x_{c}}(\tilde{q},\widetilde{\mathtt{P}},x_c)
\end{bmatrix}.
\end{equation*}
Hence, since $J(q,\mathtt{P})$ is skew-symmetric, 
\begin{equation}\label{eq:hamiltonian1dot}
    \dot{\widetilde{H}}_d=-\left(\frac{\partial \widetilde{H}_d}{\partial x_c}(\tilde{q},\widetilde{\mathtt{P}},x_c)\right)^\top R_c \left(\frac{\partial \widetilde{H}_d}{\partial x_c}(\tilde{q},\widetilde{\mathtt{P}},x_c)\right)\leq 0.
\end{equation}
Note that $\widetilde{H}_{d}(\tilde{q},\widetilde{\mathtt{P}},x_c)$ is \textit{radially unbounded}, which, in combination with \eqref{eq:hamiltonian1dot}, ensures that $\tilde{q},\widetilde{\mathtt{P}}$, and $x_{c}$ are \textit{bounded}. Thus, since we consider that $q_{d}(t)$ is bounded, we get that $q, \ \mathtt{P}$, and $z$ are bounded. Thus, it follows from Lemma \ref{lem1} that $\lVert \Psi(q)\rVert<\infty$ and $\lVert J(q,\mathtt{P})\rVert<\infty$, which with the boundedness of the state and the pH structure of the closed-loop system, imply that $\dot{\tilde{q}}, \ \dot{\widetilde{\mathtt{P}}}, \dot{x}_{c}$, and $\dot{z}$ are bounded as well. Now, differentiating the dynamics of $\tilde{q}$, we get
\begin{equation*}
 \ddot{\tilde{q}} = \left( \displaystyle\frac{d\Psi}{dt}(q) \right)\widetilde{\mathtt{P}} + \Psi(q)\dot{\widetilde{\mathtt{P}}},
\end{equation*} 
which, invoking Lemma \ref{lem1}, is bounded. Accordingly, $\dot{\tilde{q}}$ is \textit{uniformly continuous}, and it follows from Barbalat's lemma that $\dot{\tilde{q}}\to \mathbf{0}_{n}$ as $t\to\infty$. Furthermore,
\begin{equation}
 \dot{\tilde{q}}\to \mathbf{0}_{n} \implies \widetilde{\mathtt{P}}\to\mathbf{0}_{n} \label{Ptozero}
\end{equation} 
as $t\to\infty$.

Differentiating the dynamics of $\widetilde{\mathtt{P}}$, we obtain
\begin{equation}
 \ddot{\widetilde{\mathtt{P}}} = -\left( \displaystyle\frac{d\Psi^{\top}}{dt}(q) \right)K_{I}z - \Psi(q)K_{I}\dot{z}+J(q,\mathtt{P})\dot{\widetilde{P}},
\end{equation} 
where we used \eqref{Ptozero}. Therefore, $\ddot{\widetilde{\mathtt{P}}}$ is bounded, and consequently, $\dot{\widetilde{\mathtt{P}}}$ is uniformly continuous. Thus, applying Barbalat's lemma, we get that $\dot{\widetilde{\mathtt{P}}}\to\mathbf{0}_{n}$ as $t\to\infty$. Furthermore, we can establish the following chain of implications.
\begin{equation}
\arraycolsep=1.6pt
\def\arraystretch{1.5}
 \begin{array}{rcccl}
  \dot{\widetilde{\mathtt{P}}}\to\mathbf{0}_{n} &\implies& \Psi^{\top}(q)K_{I}z\to\mathbf{0} &\implies& z\to\mathbf{0}_{n},
 \end{array}\label{ztozero}
\end{equation} 
as $t\to\infty$. Now, note that, since $\dot{z}$ and $\dot{x}_{c}$ are bounded, $\dot{\tilde{H}}_{d}$ is uniformly continuous. Thus, again, invoking Barbalat's lemma we get the following
\begin{equation*}
 \begin{array}{rcccl}
  \dot{\tilde{H}}_{d}\to 0 &\implies& K_{c}x_{c}+K_{I}z \to \mathbf{0}_{n} & \implies & x_{c}\to\mathbf{0}_{n} 
 \end{array}
\end{equation*} 
as $t\to\infty$, where we used \eqref{ztozero}. Furthermore, substituting $x_{c}\to\mathbf{0}_{n}$ into \eqref{ztozero}, we get $\tilde{q}\to\mathbf{0}_{n}$ as $t\to\infty$. 
%
\end{proof}

\subsection{Saturated control without velocity measurements}

In this subsection, we modify the controller \eqref{eq:uhat1} to ensure that the control signals comply with \textbf{O2}, given in the problem formulation. Towards this end, consider the controller state $x_{c}\in\rea^{n}$ with dynamics
\begin{equation}
    \dot{x}_c=-R_c\left(\sum_{i=1}^{n} e_{i} \alpha_{i} \tanh \left(\beta_{i} z_{i}\right)+K_c x_c\right), \; i=1,\dots,n;\label{eq:xc2}
\end{equation}
where $e_{i}$ denotes an element of the canonical basis of $\rea^{n}$, the constant parameters $\alpha_{i},\beta_{i}$ are positive, the matrices $R_{c},K_{c}\in\rea^{n\times n}$ are positive definite, and $z$ is defined as in \eqref{eq:z}. The following proposition provides a saturated control law that addresses the global uniform asymptotic stabilization problem of \eqref{eq:systemerror}, which does not require velocity measurements.

\begin{proposition}\label{controllaw2}
Consider the augmented state vector $[\tilde{q}^{\top}, \widetilde{\mathtt{P}}^{\top},x_{c}^{\top}]^{\top}$, with dynamics \eqref{eq:systemerror1}-\eqref{eq:xc2}. Then, the control law
\begin{equation}\label{eq:uhat2}
    \hat{u}=-\sum_{i=1}^{n} e_{i} \alpha_{i} \tanh \left(\beta_{i} z_{i}\right)
\end{equation}
ensures that the closed-loop system has a globally uniformly asymptotically stable equilibrium at $(\tilde{q},\widetilde{\mathtt{P}},x_c) =(\mathbf{0}_{n},\mathbf{0}_{n},\mathbf{0}_{n})$ with Lyapunov function
\begin{equation}\label{eq:hamiltonian2}
    \widetilde{H}_{\tt sat}(\tilde{q},\widetilde{\mathtt{P}},x_c)=\sum_{i=1}^{n} \frac{\alpha_{i}}{\beta_{i}} \ln \left(\cosh \left(\beta_{i} z_{i}\right)\right) + \frac{1}{2}\widetilde{\mathtt{P}}^\top \widetilde{\mathtt{P}} + \frac{1}{2}x_c^\top K_c x_c.
\end{equation}
\end{proposition}
\begin{proof}
Note that $\widetilde{H}_{\tt sat}(\tilde{q},\widetilde{\mathtt{P}},x_c)$ is positive definite with respect to the equilibrium point and is \textit{radially unbounded}. Moreover, the closed-loop system takes the form
\begin{equation*}
 \begin{bmatrix}
  \dot{\tilde{q}} \\[0.15cm] \dot{\widetilde{\mathtt{P}}} \\[0.15cm] \dot{x}_{c} 
 \end{bmatrix}  =  \begin{bmatrix}
{\mathbf{0}_{n\times n}} & {\Psi (q)} & {\mathbf{0}_{n\times n}} \\[0.15cm] -\Psi^\top (q) & J(q,\mathtt{P}) & {\mathbf{0}_{n\times n}} \\[0.15cm] {\mathbf{0}_{n\times n}} & {\mathbf{0}_{n\times n}} & -R_{c}
                                        \end{bmatrix}
\begin{bmatrix}
 \frac{\partial \widetilde{H}_{\tt sat}}{\partial \tilde{q}}(\tilde{q},\widetilde{\mathtt{P}},x_c) \\[0.15cm] \frac{\partial \widetilde{H}_{\tt sat}}{\partial \widetilde{\mathtt{P}}}(\tilde{q},\widetilde{\mathtt{P}},x_c) \\[0.15cm] \frac{\partial \widetilde{H}_{\tt sat}}{\partial x_{c}}(\tilde{q},\widetilde{\mathtt{P}},x_c)
\end{bmatrix},
\end{equation*}
and
\begin{equation}\label{eq:hamiltonian2dot}
    \dot{\widetilde{H}}_{\tt sat}=-\left(\frac{\partial \widetilde{H}_{\tt sat}}{\partial x_c}(\tilde{q},\widetilde{\mathtt{P}},x_c)\right)^\top R_c \left(\frac{\partial \widetilde{H}_{\tt sat}}{\partial x_c}(\tilde{q},\widetilde{\mathtt{P}},x_c)\right)\leq 0.
\end{equation}
The rest of the proof follows from the same arguments employed in the proof of Proposition \ref{controllaw1} noting that
\begin{equation*}
  \displaystyle\frac{\partial \widetilde{H}_{\tt sat}}{\partial \tilde{q}}(\tilde{q},\widetilde{\mathtt{P}},x_c) = \displaystyle\sum_{i=1}^{n} e_{i} \alpha_{i} \tanh \left(\beta_{i} z_{i}\right) = \mathbf{0}_{n} \iff z=\mathbf{0}_{n},
\end{equation*} 
and
\begin{equation*}
 \displaystyle\frac{d}{dt}\left( \displaystyle\frac{\partial \widetilde{H}_{\tt sat}}{\partial \tilde{q}}(\tilde{q},\widetilde{\mathtt{P}},x_c) \right) = \displaystyle\sum_{i=1}^{n} e_i e_i^\top \alpha_{i}\beta_{i}\left[\sech\left(\beta_{i} z_{i}\right)\right]^2\dot{z}
\end{equation*} 
is bounded if $\dot{z}$ is bounded. 

\end{proof}

\subsection{Passivity-based trajectory tracking controller}

Proposition \ref{pro:track} introduces main result of this paper, namely, a control law that addresses the trajectory tracking problem for \eqref{eq:systemtransform} while verifying \textbf{O1} and \textbf{O2}.
\begin{proposition}\label{pro:track}
 Consider the pH system \eqref{eq:systemtransform} in closed-loop with 
 \begin{equation}\label{eq:finalcontrollaw}
 \arraycolsep=1.6pt
\def\arraystretch{1.5}
 \begin{array}{rcl}
     u&=&\displaystyle\frac{\partial V}{\partial q}(q) - \sum_{i=1}^{n} e_{i} \alpha_{i} \tanh \left(\beta_{i} z_{i}\right)+\Psi^{-\top}(q_d)\left[\displaystyle\frac{d }{d t}\left(\Psi^{-1}(q_d)\dot{q}_d\right) - J_{d}(t) \Psi^{-1}(q_d)\dot{q}_d\right].
 \end{array}
\end{equation}
Then,
\begin{equation*}
 \begin{array}{rl}
  \displaystyle\lim_{t\to\infty}q(t) = q_{d}(t), & \displaystyle\lim_{t\to\infty}\mathtt{P}(t) = \mathtt{P}_{d}(t).
 \end{array}
\end{equation*} 
\end{proposition}
\begin{proof}
 The proof follows from \eqref{eq:ud}, \eqref{eq:utilde}, Proposition \ref{controllaw2}, and noting that 
 \begin{equation*}
  \begin{array}{rl}
   \tilde{q} \to 0 \implies q\to q_{d}, & \widetilde{\mathtt{P}} \to 0 \implies \mathtt{P}\to \mathtt{P}_{d}.
  \end{array}
 \end{equation*} 
\end{proof}
\begin{remark}
 For robotic arms, the control law \eqref{eq:finalcontrollaw} can be physically interpreted as follows. The gradient of the potential energy compensates the gravitational forces acting on the system. The second term of the right-hand ensures that the trajectories of the system converge towards the desired ones, and the last term guarantees that the system keeps tracking such trajectories.  
\end{remark}
\begin{remark}
 The shape of the function $\tanh(\cdot)$, the fact that the desired trajectories are bounded, and Assumption \ref{ass1} ensure that the control law \eqref{eq:finalcontrollaw} is saturated, where the parameters $\alpha_{i}$ can be adjusted to comply with \textbf{O2}.
\end{remark}

\section{Implementation in the PERA system}\label{sec:implementation}

\begin{figure}[h]
 \centering
 \includegraphics[scale=0.2]{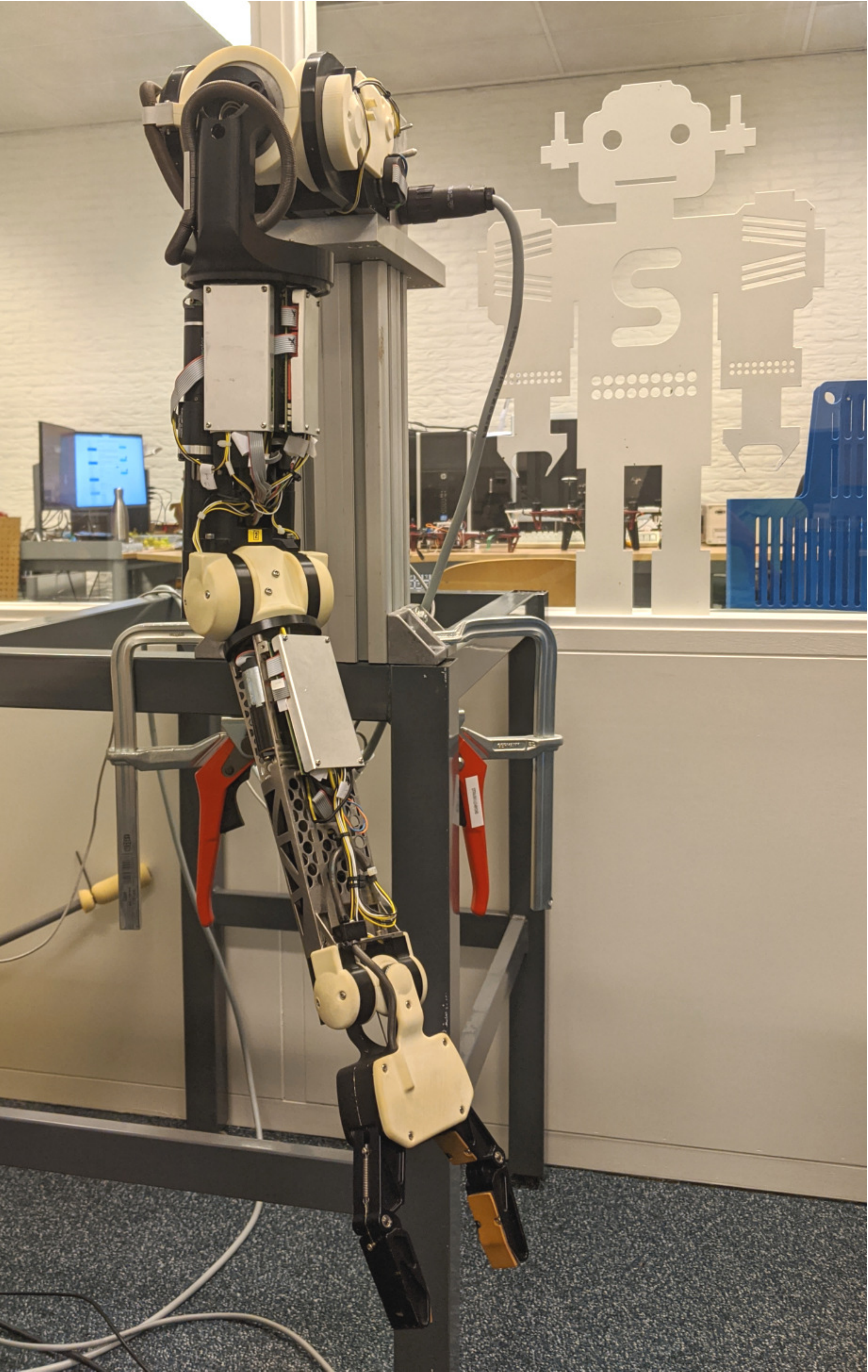}
 \caption{PERA system}
 \label{fig:PERA}
\end{figure}

To corroborate the effectiveness of the methodology proposed in the previous section, we implement the controller \eqref{eq:finalcontrollaw} in the PERA system, depicted in Fig. \ref{fig:PERA}, which is a robotic arm with seven DoF that intends to emulate the motion of a human arm. To illustrate the applicability of the saturated tracking controller, we consider only three degrees of freedom, namely, the shoulder pitch $q_{1}$, shoulder yaw $q_{2}$, and elbow pitch $q_{3}$. Accordingly, the dynamics of the reduced system can be expressed as a pH system of the form \eqref{eq:porthamiltonianframework}, with $n=3$, and
\begin{equation*}
 \arraycolsep=1.6pt
\def\arraystretch{1.5}
 \begin{array}{rcl}
  M(q)&: =& \begin{bmatrix}
           \displaystyle\sum_{i=1}^3 \mathcal{I}_{i} + a\sin^{2}(q_{1}) & 0 & \mathcal{I}_{3}\cos(q_{1}) \\ 0 & \displaystyle\sum_{j=2}^3 \mathcal{I}_{j}+a & 0 \\ \mathcal{I}_{3}\cos(q_1) & 0 & \mathcal{I}_{3} 
          \end{bmatrix} \\
  V(q)&: =& - \left( \frac{1}{3}m_{1}+m_{2} \right)gL_{1}\cos(q_{1})+\frac{1}{3}m_{2}gL_{2}b(q)\\
  a& := & (m_{1}+m_{2})L_{1}^{2} \\
  b(q)& : = & \cos(q_2)\sin(q_1)\sin(q_3) - \cos(q_1)\cos(q_3),
 \end{array}
\end{equation*} 
where the constant parameters of the system are provided in Table \ref{table1}. Moreover, the saturation limit of the motors are
\begin{equation}
 \begin{array}{rcl}
  \lvert u_{1} \rvert \leq 18.77 , & \lvert u_{2} \rvert \leq 3.32, &  \lvert u_{3} \rvert \leq 7.72.
 \end{array}\label{limits}
\end{equation} 
For further details about the PERA system, we refer the reader to \cite{rijs2010philips}.
\begin{table}[h]
  \centering
  \caption{System parameters}
  \label{table1}
\begin{center}
\begin{tabular}[t]{|c|c|c|c|}\hline
$g=9.81$  & $L_{1}=0.32$ & $L_{2}=0.48$ & $m_{1}=2.9$\\\hline
 $m_{2}=1$ & $\mathcal{I}_{1}=0.03$ &
 $\mathcal{I}_{2}=4\times 10^{-3}$
 & $\mathcal{I}_{3}=0.02$\\\hline
\end{tabular}
\end{center}
\end{table}

The control objective is to track a circular trajectory with the end-effector of the system. To this end, we parameterize the desired trajectory as follows
\begin{equation}\label{eq:desiredtrajectory}
    q_d(t)=\left[\begin{array}{c}
        0 \\[6pt]
        \arcsin\left(\frac{r}{L_2}\right)\sin\left(\frac{2\pi}{T}t\right)\\[6pt]
        \frac{\pi}{2}-\arcsin\left(\frac{r}{L_2}\right)\cos\left(\frac{2\pi}{T}t\right)
    \end{array}\right].
\end{equation}
with $r \in \mathbb{R}_+$ the radius of the circle, $T \in \mathbb{R}_+$ the period of the circle trajectory and $t \in \mathbb{R}_+$ the time.

To visualize the trajectory of the end-effector we use MATLAB\textsuperscript{\textregistered} and the Robotic Toolbox developed in \cite{corke1996robotics}, as it is illustrated in Fig. \ref{fig:desiredtrajectory}.
%
\begin{figure}[h!]
     \centering
     \begin{subfigure}[b]{0.45\textwidth}
         \centering
 \includegraphics[width=\textwidth]{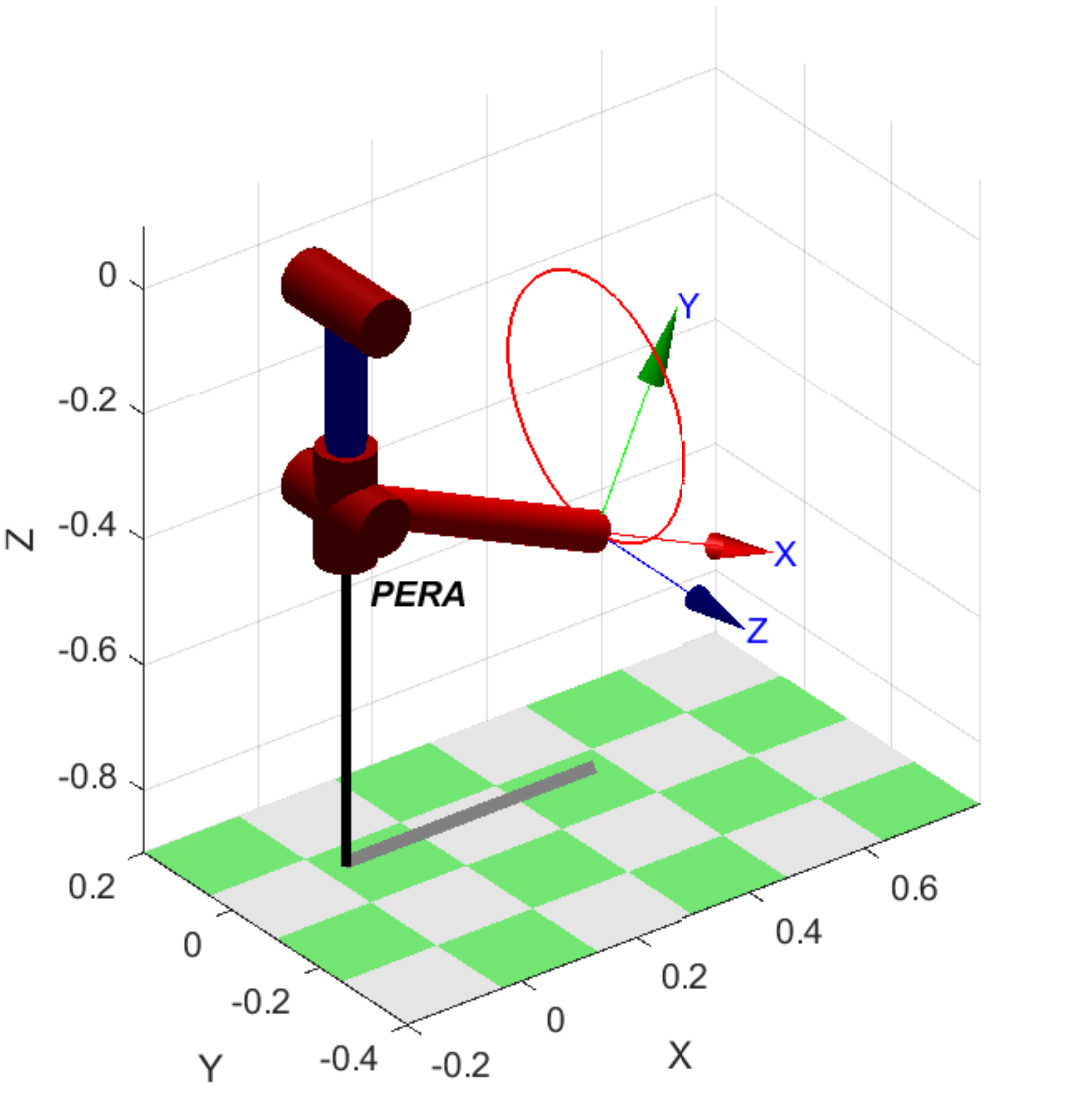}
\caption{3D view}
         \label{fig:desiredtrajectory1}
     \end{subfigure}
     \hfill
     \begin{subfigure}[b]{0.45\textwidth}
         \centering
   \includegraphics[width=\textwidth]{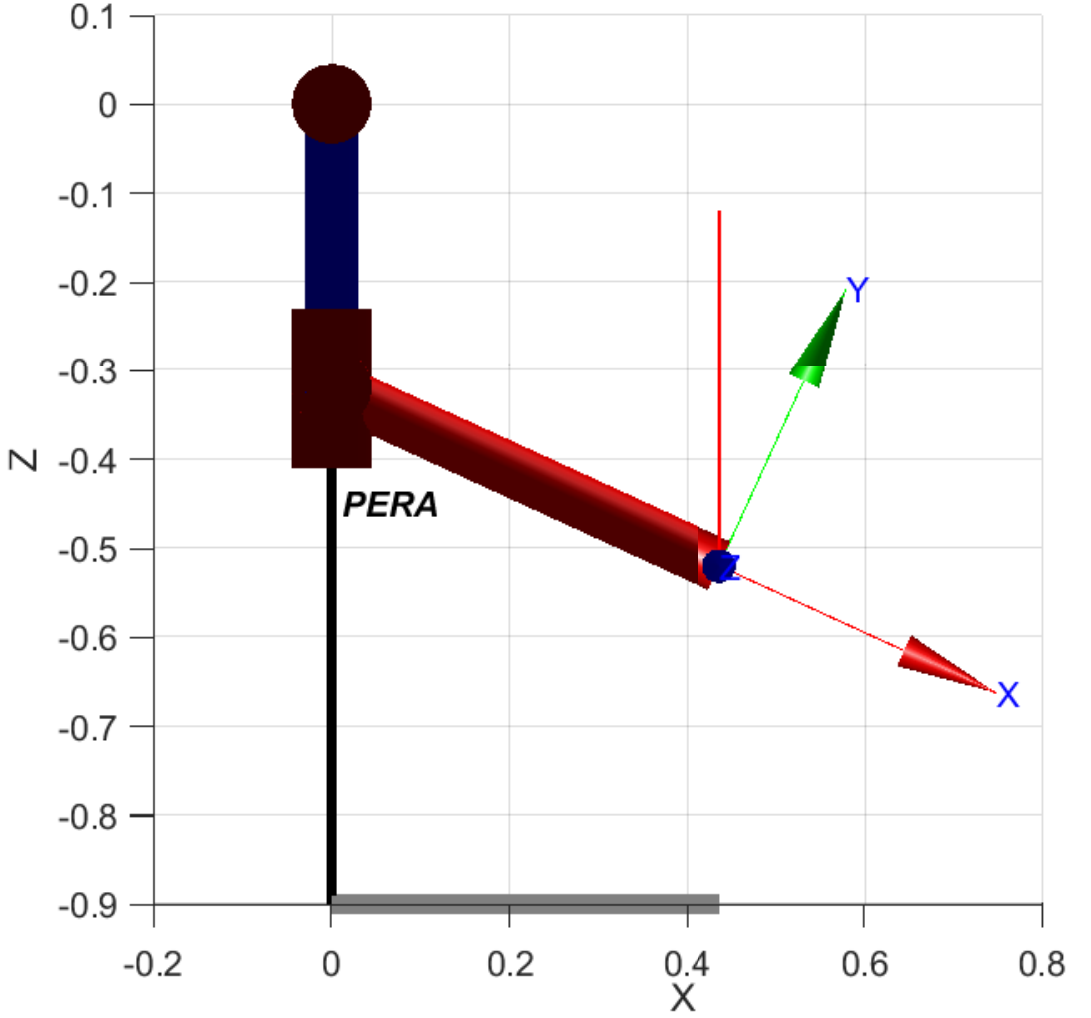}
   \caption{Side view}
         \label{fig:desiredtrajectory2}
     \end{subfigure}
        \caption{Visualization of the desired trajectory of the PERA ($r=0.2\ [m]$)}
        \label{fig:desiredtrajectory}
\end{figure}

\subsection{Simulations}
Before testing the controller on the PERA system, simulations are carried out to verify the stability of the closed-loop system. To simplify the gains selection, we select $K_{c}$ and $R_{c}$ as diagonal matrices, while the values of $\alpha$ are set such that the saturation limit of the motors cannot be reached. The proposed control parameters are
\begin{equation*}
    \begin{array}{l}
    \alpha = \left[\begin{array}{c}
        11.0 \\[6pt]
        1.7 \\[6pt]
        6.0
    \end{array}\right], \; K_c=\text{diag}\{1,\ 2,\ 0.1\} \\[0.8cm] \beta=\left[\begin{array}{c}
        40 \\[6pt]
        30 \\[6pt]
        30
    \end{array}\right], \; R_c=\text{diag}\{0.4,\ 0.11,\ 0.5\}.
    \end{array}
\end{equation*}
To perform the simulations, we consider the initial conditions $q_0=p_{0}=\mathbf{0}_3$. The results of such simulations are plotted in Fig. \ref{fig:simulation}, where it can be observed that the trajectory tracking objective is achieved.
\begin{figure}[h]
\centering
\includegraphics[width=0.7\textwidth]{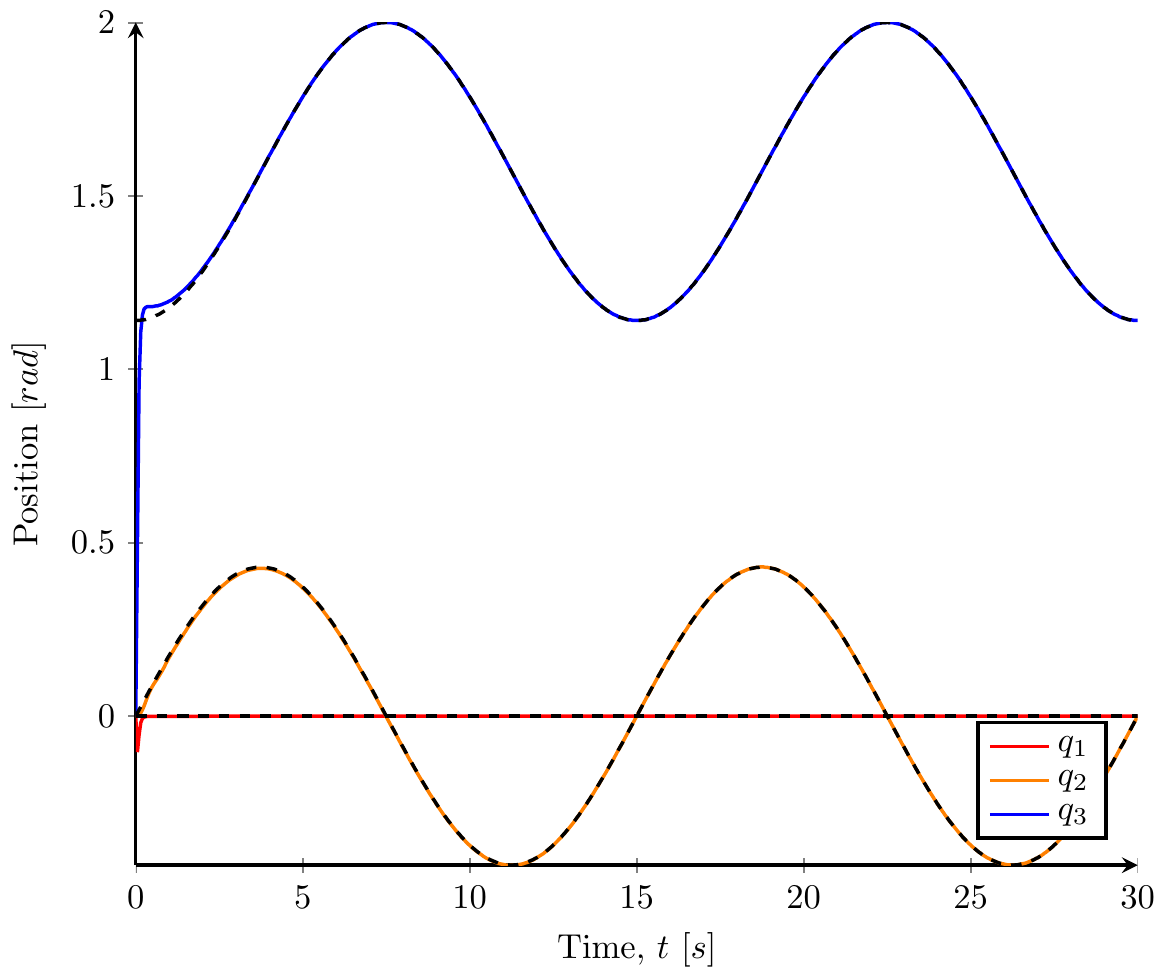}
\caption{Simulation results. The solid lines represent the trajectories described by the joints, while dashed lines represent the desired trajectories.}
\label{fig:simulation}
\end{figure}

%
%
%

\subsection{Experiments}
In contrast to the simulations, during the practical implementation of the controller, the selection of the control gains play a crucial role to guarantee an appropriate performance of the closed-loop system. Given the amount of control parameters to be tuned, the trial and error approach to select the gains proved to be challenging. The control parameters are selected as
\begin{equation}\label{eq:gaintuning}
    \begin{array}{l}
    \alpha = \left[\begin{array}{c}
        11.0 \\[6pt]
        2.0 \\[6pt]
        6.0
    \end{array}\right], \; K_c=\text{diag}\{30,\ 20,\ 200\}, \\[.8cm]\beta=\left[\begin{array}{c}
        400 \\[6pt]
        100 \\[6pt]
        120
    \end{array}\right], \; R_c=\text{diag}\{1,\ 0.1,\ 4500\}\times 10^{-4}. 
    \end{array}
\end{equation}
where the values of $\alpha$ are, again, chosen such that the limits provided in \eqref{limits} are not reached. The initial conditions of the experiments are the same as in the simulations. However, to ensure that the control task is executed correctly, we slightly modify the desired trajectory such that the end-effector is driven towards the circular trajectory, and then it starts to track it. The experimental results are shown in Fig. \ref{fig:experiments}, where it can be noticed that the system tracks the desired trajectory with a small deviation, particularly notorius in $q_{3}$. This error in the trajectory may be caused by several non-modeled phenomena that affect the behavior of the system, e.g., the friction in the joints or the anti-symmetry in the motors. Nevertheless, as it is depicted in the first column of Fig. \ref{fig:plots}, the absolute position error remains smaller than $0.05\ [rad]$, i.e. $\left|\tilde{q}\right|<0.05$. Furthermore, the control signals do not exceed the limits given in \eqref{limits} as it shown in the second column of Fig. \ref{fig:plots}, where the mentioned limits are plotted in gray dashed lines. A video of the experiments may be found in https://www.youtube.com/watch?v=2bW4PwwSo2s.
%
\begin{figure}[h]
\centering
\includegraphics[width=0.7\textwidth]{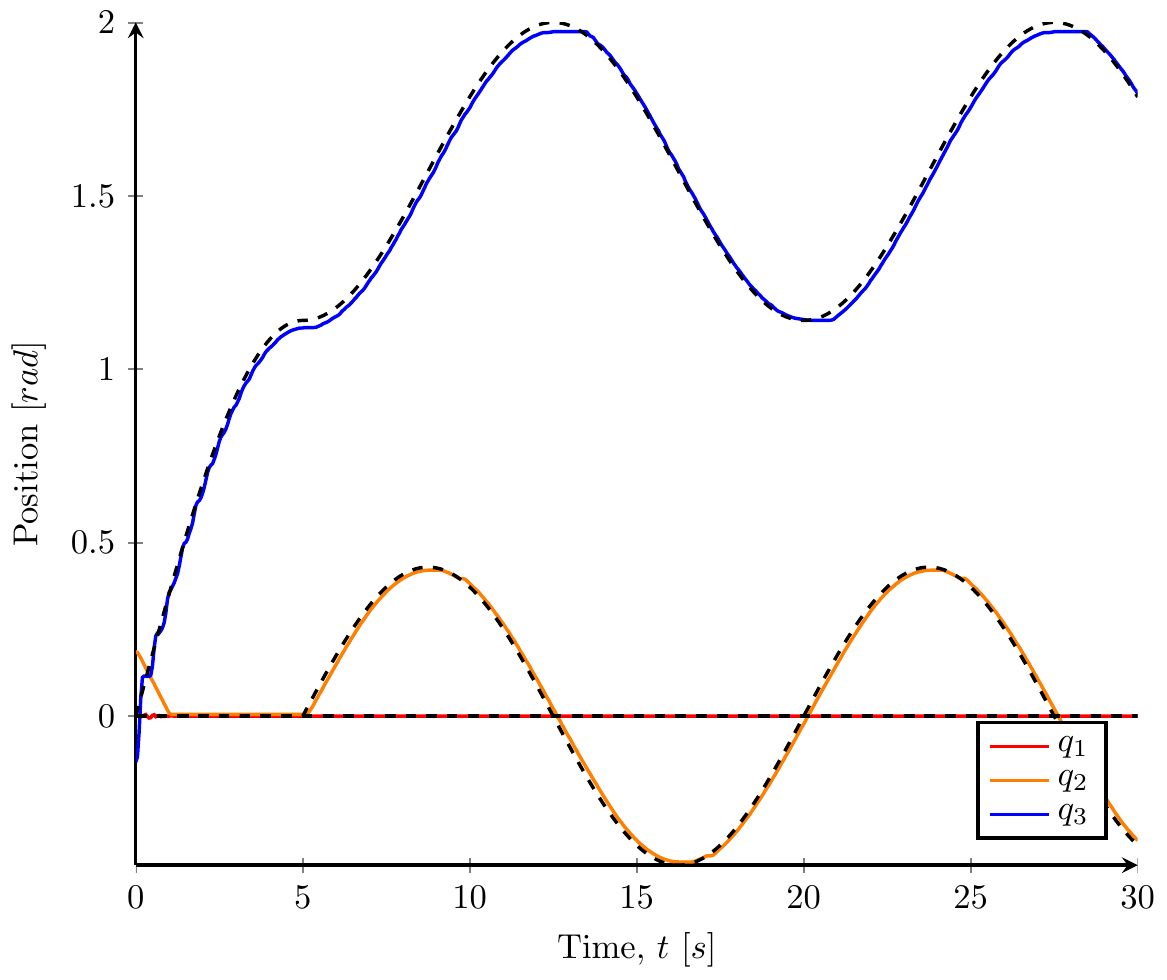}
\caption{Experimental results. The solid lines represent the trajectory of the joints during the experiment, while the dashed lines represent the desired trajectories.}
\label{fig:experiments}
\end{figure}
%
\begin{figure}[h!]
     \centering
\includegraphics[width=0.95\textwidth]{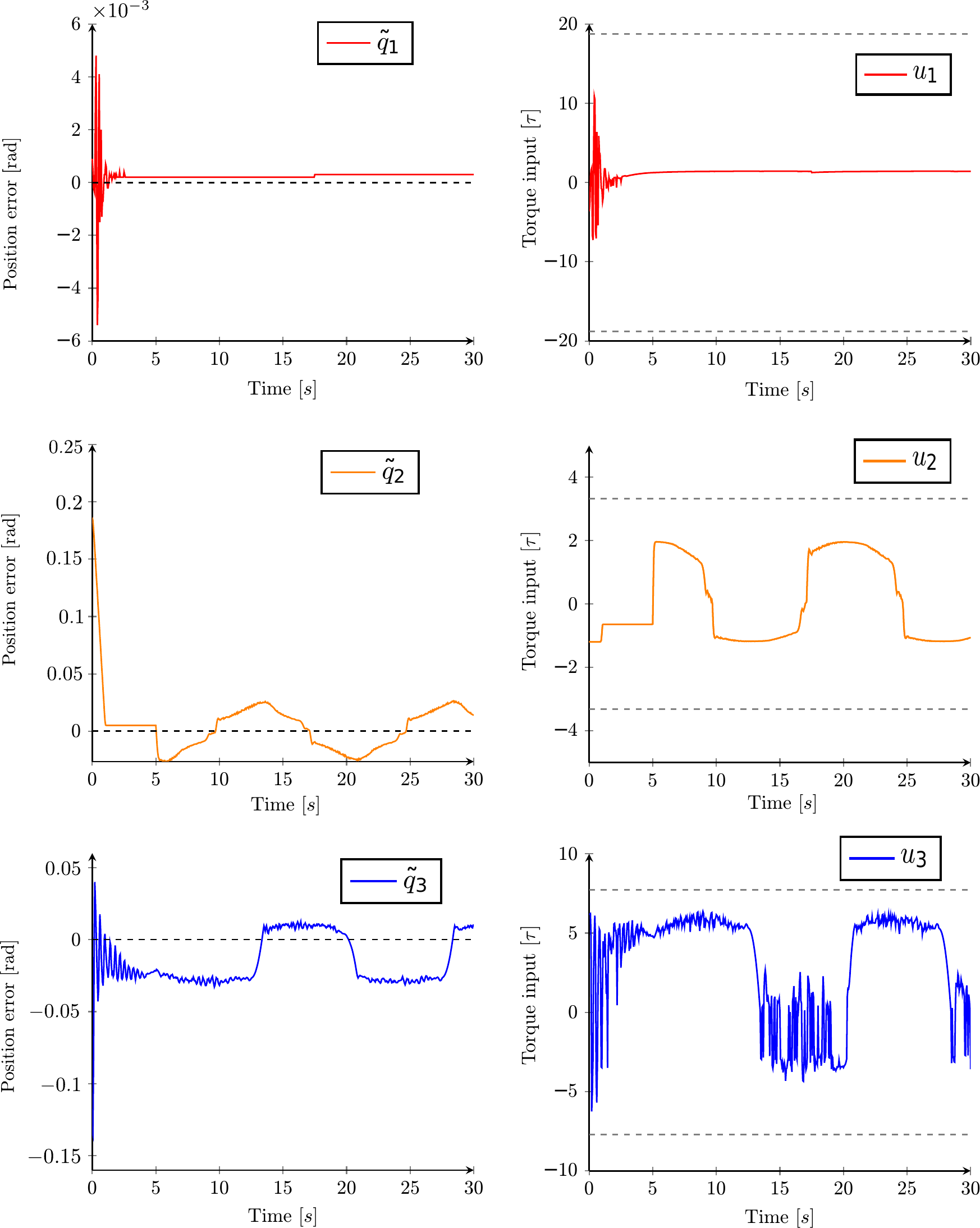}
        \caption{Position errors (first column) and control signals (second column).}
        \label{fig:plots}
\end{figure}

\section{Concluding remarks and future work}\label{sec:conclusion}
%
%

This paper presents a constructive control design methodology that solves the trajectory tracking problem for a class of robotic arms, where the control signals are saturated and do not require velocity measurements. Moreover, the control law uses gravity compensation based on the modeling of the gravitational force acting on the robotic arm. To prove that the trajectories of the closed-loop system \textit{globally uniformly asymptotically} converge to the desired trajectories, we conduct an analysis based on Barbalat's lemma. 

The control approach was implemented in the PERA system, where the experimental results show that the trajectories of the system track the desired ones with an absolute position error of the joints that remains smaller than $0.05\ [rad]$.  Additionally, the controller proved to be robust in presence of non-modeled phenomena such as the natural dissipation present in the joints of the system. The tuning of the control gains was carried out through a trial and error process, being the most challenging part of the controller implementation. Therefore, as future work, it is suggested to further investigate a systematic method for tuning of control gains. Another option is to investigate the possibility of including variable gains in the controller to improve its performance.

\bibliographystyle{ieeetr}
\bibliography{ref} 

\end{document}